\documentclass[journaldraft]{IEEEtran}
\usepackage{algorithm}
\usepackage{algorithmicx}
\usepackage{algpseudocode}
\usepackage{pifont}
\usepackage{graphicx}
\usepackage[cmex10]{amsmath}
\usepackage{multirow}
\usepackage{supertabular}
\usepackage{amsmath,amsthm}
\usepackage{amssymb}
\newtheorem{Pre}{Lemma}

\newtheorem{Performance}{Theorem}
\usepackage{cite}


\algdef{SE}[FOR]{NoDoFor}{EndFor}[1]{\algorithmicfor\ #1}{\algorithmicend\ \algorithmicfor}%
\algdef{SE}[IF]{NoThenIf}{EndIf}[1]{\algorithmicif\ #1}{\algorithmicend\ \algorithmicif}%
\algdef{SE}[WHILE]{NoDoWhile}{EndWhile}[1]{\algorithmicwhile\ #1}{\algorithmicend\ \algorithmicwhile}%

\begin{document}
\title{Streaming Algorithms for News and Scientific Literature Recommendation: Submodular Maximization with a $d$-Knapsack Constraint}

\author{Qilian Yu,~\IEEEmembership{Student Member,~IEEE,} Easton~Li~Xu,~\IEEEmembership{Member,~IEEE,} and~Shuguang~Cui,~\IEEEmembership{Fellow,~IEEE}\thanks{
Q.~Yu, E.~L.~Xu, and~S.~Cui are with the Department of Electrical and Computer Engineering, Texas
A\&M University, College Station, TX 77843 USA (e-mails: \{yuql216, eastonlixu, cui\}@tamu.edu).}}

\date{\today}
\maketitle

\begin{abstract}
Submodular maximization problems belong to the family of combinatorial optimization problems and enjoy wide applications. In this paper, we focus on the problem of maximizing a monotone submodular function subject to a $d$-knapsack constraint, for which we propose a streaming algorithm that achieves a $\left(\frac{1}{1+2d}-\epsilon\right)$-approximation of the optimal value, while it only needs one single pass through the dataset without storing all the data in the memory. In our experiments, we extensively evaluate the effectiveness of our proposed algorithm via two applications: news recommendation and scientific literature recommendation. It is observed that the proposed streaming algorithm achieves both execution speedup and memory saving by several orders of magnitude, compared with existing approaches.
\end{abstract}

\section{Introduction}
As our society enters the big data era, the main problem that data scientists are facing is how to process the unprecedented large datasets. Besides, data sources are heterogenous, comprising documents, images, sounds, and videos. Such challenges require the data processing algorithms to be more computationally efficient. The concept of submodularity plays an important role in pursuing efficient solutions for combinatorial optimization, since it has rich theoretical and practical features. Hence submodular optimization has been adopted to preprocess massive data in order to reduce the computational complexity. For example, in the kernel-based machine learning \cite{liu2013submodular,lin2009select}, the most representative subset of data is first selected in order to decrease the dimension of the feature space, by solving a submodular maximization problem under a cardinality constraint. Besides, such submodular optimization models have been extended to address data summarization problems~\cite{chakraborty2015adaptive}.

Although maximizing a submodular function under a cardinality constraint is a typical NP-hard problem, a simple greedy algorithm developed in \cite{nemhauser1978analysis} achieves a $(1-e^{-1})$-approximation of the optimal solution with a much lower computation complexity. When the main memory can store the whole dataset, such a  greedy algorithm can be easily applied for various applications. However, as it requires the full access to the whole dataset, large-scale problems prevent the greedy algorithm from being adequate due to practical computation resource and memory limitations. Even in the case when the memory size is not an issue, it is possible that the number of data samples grows rapidly such that the main memory is not able to read all of them simultaneously.

Under the scenarios discussed above, processing data in a streaming fashion becomes a necessity, where at any time point, the streaming algorithm needs to store just a small portion of data into the main memory, and produces the solution right at the end of data stream. A streaming algorithm does not require the full access to the whole dataset, thus only needs limited computation resource. In \cite{badanidiyuru2014streaming}, the authors introduced a streaming algorithm to maximize a submodular function under a cardinality constraint, where the cardinality constraint is just a special case of a $d$-knapsack constraint \cite{kulik2009maximizing} with each weight being one. When each element has multiple weights or there are more than one knapsack constraints, the algorithm proposed in \cite{badanidiyuru2014streaming} is no longer applicable.

In this paper, we develop a new streaming algorithm to maximize a monotone submodular function, subject to a general $d$-knapsack constraint. It requires only one single pass through the data, and produces a $\left(\frac{1}{1+2d} - \epsilon\right)$-approximation of the optimal solution, for any $\epsilon > 0$. In addition, the algorithm only requires $O\left(\frac{b\log b}{d\epsilon}\right)$ memory (independent of the dataset size) and $O\left(\frac{\log b}{\epsilon}\right)$ computation per element with $b$ being the standardized $d$-knapsack capacity. To our knowledge, it is the first streaming algorithm that provides a constant-factor approximation guarantee with only monotone submodularity assumed. In our experiments, compared with the classical greedy algorithm developed in \cite{sviridenko2004note}, the proposed streaming algorithm achieves over 10,000 times running time reduction with a similar performance.

The rest of this paper is organized as follows. In Section II we introduce the formulation and related existing results. In Section III we describe the proposed algorithms. In Section IV we present two applications in news and scientific literature recommendations. We draw the conclusions in Section V.

\section{Formulation and Main Results}
\subsection{Problem Formulation}
Let $V=\{1,2,\ldots,n\}$ be the ground set and $f: 2^V\to [0,\infty)$ be a nonnegative set function on the subsets of $V$. For any subset $S$ of $V$, we denote the characteristic vector of $S$ by $\bold{x}_S=(x_{S,1},x_{S,2},\ldots,x_{S,n})$, where for $1\le j\le n$, $x_{S,j} = 1$, if $j\in S$; $x_{S,j} = 0$, otherwise. For $S\subseteq V$ and $r\in V$, the marginal gain of $f$ with respect to $S$ and $r$ is defined to be
$$\Delta_f(r|S) \triangleq f(S \cup \{r\}) - f(S),$$
which quantifies the increase in the utility function $f(S)$ when $r$ is added into subset $S$. A function $f$ is submodular if it satisfies that for any $A \subseteq B \subseteq V$ and $r \in V \setminus B$, the diminishing returns condition holds:
$$\Delta_f(r|B) \leq \Delta_f(r|A).$$
Also, $f$ is said to be a monotone function, if for any $S\subseteq V$ and $r\in V$, $\Delta_f(r|S) \geq 0$. For now, we adopt the common assumption that $f$ is given in terms of a black box that computes $f(S)$ for any $S \subseteq V$. In Sections~\ref{sectionspecial},~\ref{sectiongeneral},~\ref{sectiononline}, we will discuss the case when the submodular function is independent \cite{mirzasoleiman2013distributed} of the ground set $V$ (i.e., for any $S\subseteq V$, $f(S)$ depends on only $S$, not $V\setminus S$), and in Section~\ref{sectiondependent}, we will discuss the setting where the value of $f(S)$ depends on not only the subset $S$ but also the ground set $V$.

Next, we introduce the $d$-knapsack constraint. Let $\bold{b}=(b_1,b_2,\ldots,b_d)^T$ be a $d$-dimensional budget vector, where for $1\le i\le d$, $b_i > 0$ is the budget corresponding to the $i$-th resource. Let $C=(c_{i,j})$ denote a $d\times n$ matrix, whose $(i,j)$-th entry $c_{i,j} > 0$ is the weight of the element $j \in V$ with respect to the $i$-th knapsack resource constraint. Then the $d$-knapsack constraint can be expressed by $C\bold{x}_S \leq \bold{b}$. The problem for maximizing a monotone submodular function $f: 2^V \to [0,\infty)$ subject to a $d$-knapsack constraint can be formulated as
\begin{eqnarray}\label{overallproblem}
\label{d-knap-formulation}
\begin{aligned}
   &\underset{S \subseteq V}{\textrm{maximize}} &&f(S) \\
   &\textrm{subject to} && C\bold{x}_S \leq \bold{b}.
\end{aligned}
\end{eqnarray}
We aim to \textbf{MA}ximize a monotone \textbf{S}ubmodular set function subject to a $d$-\textbf{K}napsack constraint, which is called $d$-\textbf{MASK} for short. Without loss of generality, for $1\le i\le d, 1\le j\le n$, we assume that $c_{i,j}\le b_i$. That is, no entry in $C$ has a larger weight than the corresponding knapsack budget, since otherwise the corresponding element is never selected into $S$.

For the sake of simplicity, we here standardize Problem~(\ref{overallproblem}). Let $$b\triangleq \max_{1\le i\le d} b_i\textrm{ and }c'\triangleq\min_{1\le i\le d, 1\le j\le n} b c_{i,j}/b_i.$$
For $1\le i\le d$, $1\le j\le n$, we replace each $c_{i,j}$ with $b c_{i,j}/b_ic'$ and $b_i$ with $b/c'.$ We then create a new matrix $D$ by concatenating $C$ and $\bold{b}$ over columns. That is, $D=(d_{i,j})$ is a $d\times (n+1)$ matrix, such that, for $1\le i\le d$, $d_{i,j}=c_{i,j}\ge 1$ if $1\le j\le n$; $d_{i,j}=b$ if $j=n+1$. The standardized problem has the same optimal solution as Problem~(\ref{overallproblem}). In the rest of the paper, we only consider the standardized version of the $d$-MASK problem.

\subsection{Related Work and Main Results}
Submodular optimization has been regarded as a powerful tool for combinatorial massive data mining and machine learning, for which a streaming algorithm processes the dataset piece by piece and then produces an approximate solution right at the end of the data stream. This makes it quite suitable to process a massive dataset in many applications.

When $d=1$ and all entries of $C$ are ones, Problem~(\ref{overallproblem}) is equivalent to maximizing a monotone submodular function under a cardinality constraint. This optimization problem has been proved to be NP-hard \cite{nemhauser1978analysis}, and people have developed many approximation algorithms to solve this problem, among which the greedy algorithm \cite{nemhauser1978analysis} is the most popular one. Specifically, the greedy algorithm selects the element with the maximum marginal value at each step and produces a $(1-e^{-1})$-approximation guarantee with $O(kn)$ computation complexity, where $k$ is the maximum number of elements that the solution set can include, and $n$ is the number of elements in the ground set $V$. Recently, some accelerated algorithms were proposed in \cite{mirzasoleiman2015lazier, badanidiyuru2014fast}. Unfortunately, neither of them can be applied to the case when the size of the dataset is over the capacity of the main memory. A streaming algorithm was developed in \cite{badanidiyuru2014streaming} with a $(1/2-\epsilon)$-approximation of the optimal value, for any $\epsilon > 0$. This streaming algorithm does not require the full access to the dataset, and needs only one pass through the dataset. Thus it provides a practical way to process a large dataset on the fly with a low memory requirement, but not applicable under a general $d$-knapsack constraint.

Further, the authors in \cite{lin2010multi} dealt with the case when $d=1$ and each entry of $C$ can take any positive values. Maximizing a monotone submodular function under a single knapsack constraint is also called a budgeted submodular maximization problem. This problem is also NP-hard, and the authors in \cite{sviridenko2004note} suggested a greedy algorithm, which produces a $(1-e^{-1})$-approximation of the optimal value with $O(n^5)$ computation complexity. Specifically, it first enumerates all the subsets of cardinalities at most three, then greedily adds the elements with maximum marginal values per weight to every subset starting with three elements, and finally outputs the suboptimal subset. Although the solution has a $(1-e^{-1})$-approximation guarantee, the $O(n^5)$ computation cost prevents this greedy algorithm from being widely used in practice. Hence some modified versions of the greedy algorithm have been developed. The authors in \cite{lin2010multi} applied it to document summarization with a $(1-e^{-1/2})$ performance guarantee. In \cite{leskovec2007cost}, the so-called cost effective forward (CEF) algorithm for outbreak detection was proposed, which produces a solution with a $(1-e^{-1})/2$-approximation guarantee and requires only $O(Mn)$ computation complexity, where $M$ is the knapsack budget when $d=1$.

The considered $d$-MASK problem is a generalization of the above problems to maximize a submodular function under more than one budgeted constraints. A framework was proposed in \cite{kulik2009maximizing} for maximizing a submodular function subject to a $d$-knapsack constraint, which yields a $(1-e^{-1}-\epsilon)$-approximation for any $\epsilon > 0$. However, it is hard to implement this algorithm, since it involves some high-order terms with respect to the number of budgets, making it inappropriate for processing large datasets \cite{papachristoudis2015theoretical}. Later, an accelerated algorithm was developed in \cite{kumar2013fast}. It runs for $O\left(1/\delta\right)$ rounds in MapReduce \cite{dean2004mapreduce} for a constant $\delta$, and provides an $\Omega\left(1/d\right)$-approximation. However, this algorithm needs an $O(\log n)$ blowup in communication complexity among various parts. As observed in \cite{kiveris2014connected}, such a blowup decreases its applicability in practice. Note that the authors in \cite{kumar2013fast} mentioned that the MapReduce method with an $\Omega\left(1/d\right)$-approximation can be extended to execute in a streaming fashion, but did not provide any concrete algorithms and the associated analysis.

\begin{table*}[!t]
\centering
\caption{Comparison of approximation guarantees and computation costs}
\begin{tabular}{|c|c|c|c|c|}
\hline
\multirow{2}{*}{}       & \multicolumn{2}{c|}{Best Performance Known Algorithms} & \multicolumn{2}{c|}{Proposed Streaming Algorithms}                                \\ \cline{2-5} 
                        & Approx. Factor                & Comput. Cost           & Approx. Factor                         & Comput.Cost                              \\ \hline
1-Knapsack Constraint   & $1-e^{-1}$                    & $O(n^5)$               & \multirow{2}{*}{$1/(1+2d) - \epsilon$} & \multirow{2}{*}{$O(n \log b /\epsilon)$} \\ \cline{1-3}
$d$-Knapsack Constraint & $1-e^{-1}-\epsilon$           & Polynomial             &                                        &                                          \\ \hline
\end{tabular}\label{table1}
\end{table*}

Table~\ref{table1} shows the comparison among the approximation guarantees and computation costs of the aforementioned algorithms against our proposed algorithm.

To our best knowledge, this paper is the first to propose an efficient streaming algorithm for maximizing a monotone submodular function under a $d$-knapsack constraint, with 1) a constant-factor approximation guarantee, 2) no assumption on full access to the dataset, 3) execution of a single pass, 4) $O(b\log b)$ memory requirement, 5) $O(\log b)$ computation complexity per element, and 6) only assumption on monotonicity and submodularity of the objective function. In the following section, we describe the proposed algorithm in details.

\section{Streaming Algorithms for Maximizing Monotone Submodular Functions}

\subsection{Special Case: One Cardinality Constraint}\label{sectionspecial}
We first consider a special case of the $d$-MASK problem: maximizing a submodular function subject to one cardinality constraint:
\begin{eqnarray}\label{specialcaseproblem}
\label{cardinality}
\begin{aligned}
   &\underset{S \subseteq V}{\textrm{maximize}} &&f(S) \\
   &\textrm{subject to} && |S| \leq k.
\end{aligned}
\end{eqnarray}
In \cite{nemhauser1978analysis}, the authors proved this problem is NP-hard and proposed a classical greedy algorithm. At each step of the algorithm, as we explained earlier, the element with the largest marginal value is added to the solution set. This operation, in fact, reduces the ``gap'' to the optimal solution by a significant amount. Formally, if element $j$ is added to the current solution set $S$ by the greedy algorithm, the marginal value $\Delta_f(j | S)$ of this picked element should be at least above certain threshold. In \cite{badanidiyuru2014streaming}, the authors developed the so-called Sieve-Streaming algorithm, where the threshold for the marginal value is set to be $(\textrm{OPT}/2-f(S))/(k-|S|)$, where $S$ is the current solution set, $k$ is the maximum allowed number of elements in $S$, and OPT is the optimal value of the optimization problem. In our paper, for this submodular maximization problem under a single cardinality constraint, we first introduce a simple streaming algorithm under the assumption that we have the knowledge of the optimal value of the problem.

\begin{algorithm} [H]
\caption{Simple Streaming Algorithm}
\label{c}
\begin{algorithmic}[1]
\State Input: $v$ such that $\alpha \textrm{OPT} \leq v \leq \textrm{OPT}$, for some $\alpha \in (0,1]$.
\State $S:= \emptyset.$
\NoDoFor{$j := 1$  \textbf{to} $n$}
	\If {$f(S \cup \{j\}) - f(S) \geq \frac{v}{2k}$ and $|S| \leq k$}
			\State $S := S \cup \{j\}.$
	\EndIf
\EndFor \\
\Return $S.$
\end{algorithmic}
\end{algorithm}

\begin{Performance}
\label{th}
The simple streaming algorithm (Algorithm \ref{c}) produces a solution $S$ such that $$f(S) \geq \frac{\alpha}{2}\textrm{OPT}.$$
\end{Performance}

\begin{proof}
Given $v \in [\alpha \textrm{OPT}, \textrm{OPT}]$, let us discuss the following two cases.

\underline{Case 1:} $|S| = k$. For $1\le i\le k$, let $a_i$ be the element added to $S$ in the $i$-th iteration of the for-loop. Then we obtain
\begin{align*}
   f(S)&=f(\{a_1,a_2,\ldots,a_k\})\geq f(\{a_1,a_2,\ldots,a_k\})-f(\emptyset)\\
   &=\sum_{i=1}^k \big[f(\{a_1,a_2,\ldots,a_i\})-f(\{a_1,a_2,\ldots,a_{i-1}\}) \big].
\end{align*}
By the condition in Line 4 of Algorithm 1, for $1\le i\le k$, we have
$$f(\{a_1,a_2,\ldots,a_i\})-f(\{a_1,a_2,\ldots,a_{i-1}\})\geq\frac{v}{2k},$$
and hence
$$f(S)\geq \frac{v}{2k} \cdot k \geq \frac{\alpha}{2}\textrm{OPT}.$$

\underline{Case 2:} $|S| < k$. Let $\bar{S}= S^* \backslash S$, where $S^*$ is the optimal solution to the Problem (\ref{specialcaseproblem}). For each element $a \in \bar{S}$, we have
\begin{align*}
	f(S \cup \{a\}) - f(S) <\frac{v}{2k}.
\end{align*}
Since $f$ is monotone submodular, we obtain
\begin{align*}
	&f(S^*)-f(S) =f(S \cup \bar{S}) -f(S) \\
	&\leq \sum_{a\in S'}[f(S \cup \{a\}) - f(S)] < \frac{v}{2k} \cdot k \leq \frac{1}{2}f(S^*),
\end{align*}
which implies that
$$f(S) >\frac{1}{2}f(S^*)=\frac{1}{2}\textrm{OPT}\geq \frac{\alpha }{2}\textrm{OPT}.$$
\end{proof}

This simple streaming algorithm produces a solution by visiting every element in the ground set only once. But it requires the knowledge of the optimal value of the problem. Besides, when the elements have non-uniform weights, this algorithm does not work. To deal with the problem with non-uniform weights and more than one constraint, we are going to modify the greedy rule and take the weight-dependent marginal values into account in a streaming fashion.

\subsection{General Case: Multiple Knapsack Constraints}\label{sectiongeneral}
In order to get the desirable output, in this subsection, we first assume we have some knowledge of \textrm{OPT}, and then remove this assumption by estimating \textrm{OPT} based on the maximum value per weight of any single element. At the end, we will remove all assumptions to develop the final version of the streaming algorithm for the general case of a $d$-MASK problem.

Suppose that we know a value $v$ such that $\alpha \textrm{OPT} \leq v \leq \textrm{OPT}$ for some $0 < \alpha \leq 1$. That is, we know an approximation of $\textrm{OPT}$ up to a constant factor $\alpha$. We then construct the following algorithm to choose a subset $S$ with the knowledge of the optimal value of the problem.
\begin{algorithm} [H]
\caption{\textrm{OPT}-KNOWN-$d$-MASK}
\label{d-1}
\begin{algorithmic}[1]
\State Input: $v$ such that $\alpha \textrm{OPT} \leq v \leq \textrm{OPT}$, for some $\alpha \in (0,1].$
\State $S:= \emptyset.$
\NoDoFor{$j := 1$  \textbf{to} $n$}
        \If {$c_{i,j} \hspace{-0.5mm}\geq \hspace{-0.5mm}\frac{b}{2}\hspace{-0.5mm}\textrm{ and }\hspace{-0.5mm}\frac{f(\{j\})}{c_{i,j}} \hspace{-0.5mm} \geq \hspace{-0.5mm} \frac{2v}{b(1+2d)}$\hspace{-0.5mm}
for \hspace{-0.5mm} some \hspace{-0.5mm} $i \in [1,d]$}
        \State $S:=\{j\}.$
		\State \textbf{return} $S$.
	\EndIf
	\If {$\sum_{l \in S\cup \{j\}}c_{i,l} \leq b$ and $\frac{\Delta_{f}(j | S)}{c_{i,j}} \geq \frac{2v}{b(1+2d)}$ for all $i \in [1,d]$}
			\State $S:= S \cup \{ j \}.$
		\EndIf		
\EndFor \\
\Return $S$.
\end{algorithmic}
\end{algorithm}

At the beginning of the algorithm, the solution set $S$ is set to be an empty set. The algorithm will terminate when either we find an element $j\in V$ satisfying
\begin{equation}\label{defbigelement}
c_{i,j} \geq \frac{b}{2}\textrm{ and }\frac{f(\{j\})}{c_{i,j}} \geq \frac{2v}{b(1+2d)}\textrm{ for some }i \in [1,d],
\end{equation}
or we finish one pass through the dataset. Here we define that an element $j \in V$ is a \emph{big element} if it satisfies (\ref{defbigelement}). When the algorithm finds a big element $a$, it simply outputs $\{a\}$ and terminates. The following lemma shows that $\{a\}$ is already a good enough solution.

\begin{Pre}
\label{lemmabigelement}Assume the input $v$ satisfies $\alpha \textrm{OPT} \leq v \leq \textrm{OPT}$, and $V$ has at least one big element. The output $S$ of Algo-\mbox{rithm~2} satisfies
$$f(S) \geq \frac{\alpha}{1+2d}\textrm{OPT}.$$
\end{Pre}

\begin{proof}
Let $a$ be the first big element that Algorithm~2 finds. Then $\{a\}$ is output and the algorithm terminates. Therefore, by (\ref{defbigelement}), we have
$$f(S)=f(\{a\}) \geq \frac{2v}{b(1+2d)} \cdot \frac{b}{2} =\frac{v}{1+2d}\geq \frac{\alpha}{1+2d}\textrm{OPT}.$$
\end{proof}

When $V$ does not contain any big elements, during the data streaming, an element $j$ is added to the solution set $S$ if 1) the marginal value per weight for each knapsack constraint $\Delta_f(j|S)/c_{i,j}$ is at least $\beta v/b$ for $1\le i\le d$, and 2) the overall $d$-knapsack constraint is still satisfied. In this paper, we set  $\beta = \frac{2b}{1+2d}$, which gives us the best approximation guarantee as shown in the proof of Theorem~\ref{th1}. The following lemma shows the property of the output of Algorithm 2.
\begin{Pre}
\label{2p}
Assume that $V$ has no big elements. The output $S$ of Algorithm~\ref{d-1} has the following two properties:
\begin{enumerate}
  \item There exists an ordering $a_1, a_2, \ldots, a_{|S|}$ of the elements in $S$, such that for all $0 \leq t < |S|$ and $1 \leq i \leq d$, we have \begin{equation}\label{condition1}
      \frac{\Delta_f(a_{t+1}|S_t)}{c_{i,a_t}} \geq \frac{2v}{b(1+2d)},
        \end{equation}
      where $S_t = \{ a_1, a_2, \ldots, a_t\}$.
  \item  Assume that for $1 \leq i \leq d$, $\sum_{t = 1}^{|S|}c_{i,a_t} \leq b/2$. Then for each $a_j \in V$, there exists an index $\mu(a_j)$, with $1 \leq \mu(a_j) \leq d$ such that
      $$\frac{\Delta_f(a_j|S)}{c_{\mu(a_j), a_j}} < \frac{2v}{b(1+2d)}.$$
\end{enumerate}
\end{Pre}

\begin{proof}
1) For $0\le t< |S|$, at the $(t+1)$-th step of the algorithm, assume that $a_{t+1}$ is the element added to the current solution set $S_t=\{a_1,a_2,\ldots,a_t\}$. Then $a_1,a_2,\ldots,a_{|S|}$ forms an ordering satisfying (\ref{condition1}).

2) By contradiction, assume that there exists $j\in V$ such that for $1\le i\le d$, we have
$$\frac{f(S \cup \{j\})-f(S)}{c_{i, j}} \geq \frac{2v}{b(1+2d)}.$$
Since $j$ is not a big element and $f$ is submodular, we have
$c_{i,j}<b/2$, for $1\le i\le d$. Then $j$ can be added into $S$, where a contradiction occurs.
\end{proof}

We then establish the following theorem to show that Algorithm~\ref{d-1} produces an $\Big(\frac{\alpha}{1+2d}\Big)$-approximation of the optimal solution to Problem (\ref{overallproblem}).
\begin{Performance}
\label{th1}
Assuming that the input $v$ satisfies $\alpha \textrm{OPT} \leq v \leq \textrm{OPT}$, Algorithm \ref{d-1} has the following properties: 
\begin{itemize}
  \item It outputs $S$ that satisfies $f(S) \geq  \frac{\alpha}{1+2d} \textrm{OPT}$;
  \item It only goes one pass over the dataset, stores at most $O(b)$ elements, and has $O(d)$ computation complexity per element.
\end{itemize}
\end{Performance}

\begin{proof}
If $V$ contains at least one big element, by Lemma~\ref{lemmabigelement}, we have
$$f(S) \geq \frac{\alpha}{1+2d}\textrm{OPT};$$
otherwise, we discuss the following two cases:

\underline{Case 1:} $\sum_{j \in S}c_{i,j} \geq b/2$, for some $i \in [1,d]$. By the submodularity of $f$ and Property 1) in Lemma~\ref{2p}, we have
$$f(S) \geq \frac{2v}{b(1+2d)} \sum_{j \in S}c_{i,j} \geq \frac{v}{1+2d} \geq \frac{\alpha}{1+2d}\textrm{OPT}.$$

\underline{Case 2:} $\sum_{j \in S}c_{i,j} < b/2$, for all $i \in [1,d]$. Let $S_i^*$  be the set of elements $a_j \in S^* \backslash S$ such that $\mu(a_j) = i$, for $1 \leq i \leq d$. Then we have $S^*\setminus S=\bigcup_{1\le i\le d}S^*_i.$ With the help of the submodularity of $f$ and Property 2) in Lemma~\ref{2p}, we obtain
\begin{eqnarray*}
\begin{aligned}
	f(S \cup S_i^*) - f(S) \leq \frac{2v}{b(1+2d)} \sum_{a_j \in S_i^*}c_{\mu(a_j), a_j} < \frac{2v}{1+2d},
\end{aligned}
\end{eqnarray*}	
for $1\le i\le d$.
Then we have
\begin{eqnarray*}
\begin{aligned}
	&f(S^*)-f(S)=f(S \cup (S^* \setminus S)) - f(S)  \\
	&\leq \sum_{1\le i\le d}[f(S \cup S_i^*) - f(S)] <\frac{2dv}{1+2d},
\end{aligned}
\end{eqnarray*}	
and further,
\begin{eqnarray*}
\begin{aligned}
	f(S) > f(S^*) - \frac{2dv}{1+2d} \geq \frac{1}{1+2d} \textrm{OPT} .
\end{aligned}
\end{eqnarray*}	

In both cases, we conclude $$f(S)\ge \frac{\alpha}{1+2d}\textrm{OPT}.$$

Since we have $c_{i,j} \geq 1$ for all $i \in [1,d]$, $j \in [1,n]$, we store at most $O(b)$ elements during the algorithm. In the for-loop, we compare the values at most $d$ times. Then the computation cost per element in the algorithm is $O(d)$.
\end{proof}

We can obtain an approximation of the optimal value OPT by solving the $d$-MASK problem via Algorithm \ref{d-1} . But in certain scenarios, requiring the knowledge of an approximation to the optimization problem and utilizing the approximation in Algorithm \ref{d-1} lead to a chicken and egg dilemma. That is, we have to first estimate $\textrm{OPT}$ and then use it to compute $\textrm{OPT}$. Fortunately, even in such scenarios, we still have the following lemma to estimate $\textrm{OPT}$ if we know
$m \triangleq  \max_{1\le i\le d, 1\le j\le n} f(\{j\})/c_{i,j}$, the maximum value per weight of any single element.

\begin{Pre}
\label{L3}
Let
\begin{align*}
Q = \Big\{&\left.[1+(1+2d)\epsilon]^l \right.|l \in \mathbb{Z}, \\
&\frac{m}{1+(1+2d)\epsilon}\leq [1+(1+2d)\epsilon]^l \leq bm\Big\}
\end{align*}
 for some $\epsilon$ with $0 < \epsilon < \frac{1}{1+2d}$. Then there exists at least some $v \in Q$ such that $[1-(1+2d)\epsilon]\textrm{OPT} \leq v \leq \textrm{OPT}$.
\end{Pre}

\begin{proof}
First, choose $i'\in [1, d],$ $j'\in [1,n]$ such that $f(\{j'\})/c_{i',j'}=m.$ Since $c_{i',j'}\ge 1$, we have
$$\textrm{OPT}\ge f(\{j'\})=mc_{i',j'}\ge m.$$
Also, let $\{j_1,j_2,\ldots,j_t\}$ be a subset of $V$ such that $f(\{j_1,j_2,\ldots,j_t\})=\textrm{OPT}$. Then by the submodularity of $f$,
\begin{align*}
\textrm{OPT}&=f(\emptyset)+\sum_{i=1}^t [f(\{j_1,j_2,\ldots,j_i\})-f(\{j_1,j_2,\ldots,j_{i-1}\})]\\
&\le f(\emptyset)+\sum_{i=1}^t [f(\{j_i\})-f(\emptyset)]\\
&\le \sum_{i=1}^t f(\{j_i\})\le m\sum_{i=1}^tc_{1,j_i}\le bm.
\end{align*}
Setting $v = [1+(1+2d)\epsilon]^{\left\lfloor\log_{1+(1+2d)\epsilon}\textrm{OPT}\right\rfloor}$, we then obtain
$$\frac{m}{1+(1+2d)\epsilon}\le \frac{1}{1+(1+2d)\epsilon}\textrm{OPT}\le v\le \textrm{OPT}\le bm,$$
and
$$v\ge \frac{1}{1+(1+2d)\epsilon}\textrm{OPT}\ge [1-(1+2d)\epsilon]\textrm{OPT}.$$
\end{proof}

Based on Lemma~\ref{L3}, we propose the following algorithm that gets around the chick and egg dilemma.

\begin{algorithm}[H]
\caption{$m$-KNOWN-$d$-MASK}
\label{d-2}
\begin{algorithmic}[1]
\State Input: $m$.
\State $Q := \{ [1+(1+2d)\epsilon]^l | l \in\mathbb{Z}, $
\State \quad $\frac{m}{1+(1+2d)\epsilon} \leq [1+(1+2d)\epsilon]^l \leq b m\}$.
\NoDoFor{$v \in Q$}
    \State $S_v := \emptyset.$
\EndFor
\NoDoFor{$j := 1$  \textbf{to} $n$}
    \If {$c_{i,j} \hspace{-0.5mm}\geq \hspace{-0.5mm}\frac{b}{2}\hspace{-0.5mm}\textrm{ and }\hspace{-0.5mm}\frac{f(\{j\})}{c_{i,j}} \hspace{-0.5mm} \geq \hspace{-0.5mm} \frac{2v}{b(1+2d)}$\hspace{-0.5mm}
for \hspace{-0.5mm} some \hspace{-0.5mm} $i \in [1,d]$}
\State $S:=\{j\}.$		
\State \textbf{return} $S$.
	\EndIf
	\NoDoFor{$v \in Q$}
		\If {$\sum_{l \in S\cup \{j\}}c_{i,l} \leq b$ and $\frac{\Delta_{f}(j | S)}{c_{i,j}} \geq \frac{2v}{b(1+2d)}$ for all $i \in [1,d]$}
			\State $S_v:= S_v \cup \{ j \}.$
		\EndIf
	\EndFor
\EndFor
\State $S:=\underset{S_v, v \in Q} {\mathrm{argmax}} ~f(S_v)$.\\
\Return $S$.
\end{algorithmic}
\end{algorithm}

Then we establish the following theorem to show that the above algorithm achieves a $\left(\frac{1}{1+2d}-\epsilon\right)$-approximation guarantee, and requires $O\left(\frac{b\log b}{\epsilon}\right)$ memory and $O\left(\frac{\log b}{\epsilon}\right)$ computation complexity per element.

\begin{Performance}
\label{th2}
With $m$ known, Algorithm \ref{d-2} has the following properties:
\begin{itemize}
  \item It outputs $S$ that satisfies $f(S) \geq \left(\frac{1}{1+2d} - \epsilon \right)\textrm{OPT}$;
  \item It goes one pass over the dataset, stores at most $O\left(\frac{b\log b}{d\epsilon}\right)$ elements, and has $O\left(\frac{\log b}{\epsilon}\right)$ computation complexity per element.
\end{itemize}
\end{Performance}

\begin{proof}
By Lemma~\ref{L3}, we choose $v\in Q$ such that \\$[1-(1+2d)\epsilon]\textrm{OPT}\le v\le \textrm{OPT}$. Then by Theorem \ref{th1}, the output $S$ satisfies
$$f(S) \geq \frac{1-(1+2d)\epsilon}{1+2d} \textrm{OPT} = \left(\frac{1}{1+2d} - \epsilon \right)\textrm{OPT}.$$

Notice that there are at most $\left\lceil\log_{1+(1+2d)\epsilon}b\right\rceil+1$ (of order $\frac{\log b}{d \epsilon}$) elements in $Q$. At the end of the algorithm, $S_v$ with the largest function value will be picked to be the output. Since $S$ contains at most $b$ elements, Algorithm~\ref{d-2} stores at most $O\left(\frac{b\log b}{d\epsilon}\right)$ elements and has $O\left(\frac{\log b}{\epsilon}\right)$ computation complexity per element.
\end{proof}

Introducing the maximum marginal value per weight $m$ avoids the chicken and egg dilemma in Algorithm \ref{d-1}. With $m$ known, Algorithm \ref{d-2} needs only one pass over the dataset. However, we need an extra pass through the dataset to obtain the value of $m$. In the following, we will develop our final one-pass streaming algorithm with $m$ unknown.

\begin{algorithm}
\caption{$d$-KNAPSACK-STREAMING}
\label{d-3}
\begin{algorithmic}[1]
\State $Q := \{ [1+(1+2d)\varepsilon]^l | l \in\mathbb{Z}\}$.
\NoDoFor{$v \in Q$}
    \State $S_v := \emptyset.$
\EndFor
\State $m: = 0.$
\NoDoFor{$j := 1$  \textbf{to} $n$}
    \If {$c_{i,j} \hspace{-0.5mm}\geq \hspace{-0.5mm}\frac{b}{2}\hspace{-0.5mm}\textrm{ and }\hspace{-0.5mm}\frac{f(\{j\})}{c_{i,j}} \hspace{-0.5mm} \geq \hspace{-0.5mm} \frac{2v}{b(1+2d)}$\hspace{-0.5mm}
for \hspace{-0.5mm} some \hspace{-0.5mm} $i \in [1,d]$}
\State $S:=\{j\}.$	
\State \textbf{return} $S.$
	\EndIf	
\NoDoFor{$i := 1$ \textbf{to} $d$}
		\State $m := \max\{m, f(\{j\})/c_{i,j}\}$.
	\EndFor
	\State $Q := \{ [1+(1+2d)\varepsilon]^l | l \in\mathbb{Z},$
\State \quad $\frac{m}{1+(1+2d)\epsilon} \leq [1+(1+2d)\varepsilon]^l \leq 2bm \}$.
	\NoDoFor{$v \in Q$}
		\If {$\sum_{l \in S\cup \{j\}}c_{i,l} \leq b$ and $\frac{\Delta_{f}(j | S)}{c_{i,j}} \geq \frac{2v}{b(1+2d)}$ for all $i \in [1,d]$}
			\State $S_v:= S_v \cup \{ j \}.$
		\EndIf
	\EndFor
\EndFor
\State $S:=\underset{S_v, v \in Q} {\mathrm{argmax}} ~f(S_v)$.\\
\Return $S$.
\end{algorithmic}
\end{algorithm}
We modify the estimation candidate set $Q$ into $\{[1+(1+2d)\epsilon]^l | l \in\mathbb{Z}, \frac{m}{1+(1+2d)\epsilon} \leq [1+(1+2d)\epsilon]^l \leq 2bm\}$, and maintain the variable $m$ that holds the current maximum marginal value per weight of all single element. During the data streaming, if a big element $a$ is observed, the algorithm simply outputs $\{a\}$ and terminates. Otherwise, the algorithm will update $m$ and the estimation candidate set $Q$. If the marginal value per weight for each knapsack constraint $\Delta_f(j|S)/c_{i,j}$ is at least \mbox{$2v/b(1+2d)$} for $1\le i\le d$, and the overall $d$-knapsack constraint is still satisfied, then an element $j$ is added to the corresponding candidate set. Then we establish the following theorem, which shows the property of the output of Algorithm 4. Its proof follows the same lines as the proof of Theorem~\ref{th2}.

\begin{Performance}
\label{th4}
Algorithm \ref{d-3} has the following properties:
\begin{itemize}
  \item It outputs $S$ that satisfies that $f(S) \geq \left(\frac{1}{1+2d} - \epsilon \right)\textrm{OPT}$;
  \item It goes one pass over the dataset, stores at most $O\left(\frac{b\log b}{d\epsilon}\right)$ elements, and has $O\left(\frac{\log b}{\epsilon}\right)$ computation complexity per element.
\end{itemize}
\end{Performance}

\subsection{Online Bound}\label{sectiononline}
To evaluate the performance of our proposed algorithms, we need to compare the function values obtained by our streaming algorithm against OPT, by calculating their relative difference. Since OPT is unknown, we could use an upper bound of OPT to evaluate the performance of the proposed algorithms. 

By Theorem~\ref{th4}, we obtain
\begin{equation}\label{theofflinebound}
\textrm{OPT} \leq \frac{1+2d}{1-(1+2d)\epsilon} f(S).
\end{equation}
Then $\frac{1+2d}{1-(1+2d)\epsilon} f(S)$ is an upper bound of the optimal value to the $d$-MASK problem.  In most of cases, this bound is not tight enough. In the following, we provide a much tighter bound derived by the submodularity of $f$.

\begin{Performance}\label{theoremforoutbound}
Consider a subset $S \subseteq V$. For $1\le i\le d$, let $r_{i,s}=\Delta_f(s|S)/c_{i,s}$, and $s_{i,1}, \ldots, s_{i,|V\setminus S|}$ be the sequence such that $r_{i,s_{i,1}}\ge r_{i,s_{i,2}}\ge \cdots \ge r_{i,s_{i,|V\setminus S|}}.$ Let $k_i$ be the integer such that $\sum_{j = 1}^{k_i-1}c_{i,s_{i,j}} \leq b$ and $\sum_{j = 1}^{k_i}c_{i,s_{i,j}} > b$. And let $\lambda_i = \left.\left(b-\sum_{j = 1}^{k_i-1}c_{i,s_{i,j}}\right)\right/c_{i,s_{i,k_i}}$. Then we have
\begin{align}\label{outboundtheorem}
   \emph{OPT}=\max_{C\bold{x_{S'}} \leq \bold{b}}&\ f(S')\leq f(S) \nonumber \\
   +\min_{1\le i\le d} &\left[\sum_{j = 1}^{k_i-1}\Delta_f(s_{i,j}|S)+\lambda_i \Delta_f(s_{i,k_i}|S)\right].
\end{align}
\end{Performance}
\begin{proof}
Here we use a similar proof as the proof of Theor\mbox{em~8.3.3} in \cite{gomez2010inferring}, where the author deals with the submodular maximization problem under one knapsack constraint. Let $S^*$ be the optimal solution to Problem~(\ref{overallproblem}). First we consider the $1$-MASK problem, which has the same objective function as Problem~(\ref{overallproblem}) but only with the $i$-th knapsack constraint. Assume ${S_i^{*}}$ is its optimal solution. Since this $1$-MASK problem has fewer constraints than Problem~(\ref{overallproblem}), we have $f(S^*)\leq f(S_i^{*})$. Hence,
\begin{align}
\label{ex}
	f(S^*) \leq \min_{1\le i\le d}f(S_i^{*}).
\end{align}
Since $f$ is monotone submodular, for $1\le i\le d$,
\begin{align}
\label{bound2}
   f(S_i^*)\leq f(S\cup  {S_i^{*}}) \leq f(S) + \sum_{s \in  {S_i^{*}}} \Delta_f(s|S).
\end{align}
We first assume that all weights $c_{i,j}$ and knapsack $b$ are rational numbers. For the $i$-th $1$-MASK problem, we can multiply all $c_{i,j}$ and $b$ by the least common multiple of their denominators, making each weight and budget be an integer. We then replicate each element $s$ in $V$ into $c_{i,s}$ copies. Let $s'_i$ denote any one copy of $s$, and let $V_i'$ and ${S_i^{*}}'$ be the sets of the copies of all elements in $V$ and ${S_i}^*$, respectively. Also, define $\Delta'_f(s'_i|S)\triangleq \Delta_f(s|S)/c_{i,s}$. Then
\begin{align}
\label{bound3}
   \sum_{s\in  {S_i^{*}}} \Delta_f(s|S) &= \sum_{s'_i\in {S_i^{*}}'}\Delta'_f(s_i'|S) \nonumber \\
   &\leq \max_{K'\subseteq V_i', |K'| \leq b} \sum_{s_i'\in K'}\Delta'_f(s_i'|S).
\end{align}
To find the value of the right-hand side of (\ref{bound3}), we actually need to
solve a unit-cost modular optimization problem as follows. We first sort all elements $s'$ in $V_i'$ such that the
corresponding values $\Delta'_f(s'|S)$ form a non-increasing sequence. In this sequence, the first $b$ elements are $c_{i,s_{i,j}}$ copies of $s_{i,j}$ for $1\le j\le k_i-1$, and $\left(b-\sum_{j = 1}^{k_i-1}c_{i,s_{i,j}}\right)$ copies of $s_{i,k_i}$. Therefore, we obtain
\begin{align}
\label{bound4}
   \max_{\begin{subarray}{c}K'\subseteq V_i' \\ |K'| \leq b \end{subarray}} \sum_{s_i'\in K'}\Delta'_f(s_i'|S) = \sum_{j = 1}^{k_i-1}\Delta_f(s_{i,j}|S)+\lambda_i \Delta(s_{i,k_i}|S).
\end{align}
Combining (\ref{ex}), (\ref{bound2}), (\ref{bound3}) and (\ref{bound4}), we obtain (\ref{outboundtheorem}).

For irrational weights and knapsacks, let $\left\{c_{i,s_{i,j},t}\right\}_{t=1}^\infty$ and $\left\{b_{t}\right\}_{t=1}^\infty$ be two rational sequences with limits $c_{i,s_{i,j}}$ and $b$, respectively. And further let $k_{i,t}$ be the integer such that $\sum_{j = 1}^{k_{i,t}-1}c_{i,s_{i,j},t} \leq b_t$ and $\sum_{j = 1}^{k_{i,t}}c_{i,s_{i,j},t} > b_t$, and let
$$\lambda_{i,t}= \left.\left(b_t-\sum_{j = 1}^{k_{i,t}-1}c_{i,s_{i,j},t}\right)\right/c_{i,s_{i,k_{i,t}},t}.$$ Then $\left\{\lambda_{i,t}\right\}_{t=1}^\infty$ is a rational sequence with limit $\lambda_i$. According to the above argument, we obtain for each $t$,
\begin{align*}
   \max_{C\bold{x_{S'}}\leq \bold{b}} f(S')\leq & \ f(S) \\
+\min_{1\le i\le d} &\left[\sum_{j = 1}^{k_i-1}\Delta_f(s_{i,j}|S)+\lambda_{i,t} \Delta_f(s_{i,k_i}|S)\right].
\end{align*}
By letting $t$ go to infinity, we then finish the proof.
\end{proof}
A bound is called to be \emph{offline} \cite{gomez2010inferring} if it can be stated before we run the algorithm; otherwise, it is an \emph{online} one \cite{gomez2010inferring}. Here, we obtain an offline bound (\ref{theofflinebound}) and an online bound (\ref{outboundtheorem}), the latter of which can be calculated by the following algorithm.

\begin{algorithm}[H]
\caption{Online Bound of the $d$-MASK Problem}
\label{dbound}
\begin{algorithmic}[1]
\State \textbf{Input}: $S.$
\NoDoFor {$i := 1$ $\textbf{to}$ $d$}
	\State $S_i' := \emptyset$.
	\NoDoFor {$s$ \textbf{in} $V$}
    \State $r_{i,s} := \Delta_f(s|S)/c_{i,s}.$
    \EndFor
	\NoDoWhile{$\{s\in V\setminus(S\cup S_i')|\sum_{j\in{S \cup S_i'\cup\{s\}}}c_{i,j}\leq b\}\ne \emptyset$}
		\State $s' := \underset{s\in V \backslash (S \cup S_i'), \sum_{j\in{S \cup S_i'\cup\{s\}}}c_{i,j}\leq b} {\mathrm{argmax}} ~r_{i,s}$.
		\State $S_i' := S_i' \cup \{s'\}$.
	\EndWhile
	\State $s' := \underset{s\in V \backslash (S \cup S_i'), \sum_{j\in{S \cup S_i'\cup\{s\}}}c_{i,j}\leq b} {\mathrm{argmax}} ~r_{i,s}$.
	\State $\lambda_i:=(b-\sum_{s \in S_i'}c_{i,s})/c_{i,s'}$.
	\State $\delta_i := \sum_{s \in S_i'}{\Delta_f(s|S)}+\lambda_i \Delta_f(s'|S)$.
\EndFor\\
\Return $f(S)+\min_{1\le i\le d}\delta_i$.
\end{algorithmic}
\end{algorithm}

\subsection{Problems with Ground-Set Dependent Submodular Functions}\label{sectiondependent}
In the previous sections, we have discussed the case when the submodular function $f$ is independent of the ground set $V$. In the following, we will discuss the setting where $f$ is additively decomposable \cite{mirzasoleiman2013distributed}, and the value of $f(S)$ depends on not only the subset $S$ but also the ground set $V$. Here a function $f$ is called to be \emph{additively decomposable} \cite{mirzasoleiman2013distributed} over the ground set $V$, if there exists a family of functions $\{f_i\}_{i=1}^{|V|}$ with $f_i:2^V\to [0,\infty)$ independent of the ground set $V$ such that
\begin{align}
\label{indep_f}
f(S) = \frac{1}{|V|}\sum_{i \in V}f_i(S).
\end{align}
Algorithm 4 is still useful for the case when $f$ is dependent on the ground set but additively decomposable. To reduce the computational complexity, we randomly choose a small subset $\widetilde{V}$ of $V$, and use
$$f_{\widetilde{V}}(S) \triangleq \frac{1}{|\widetilde{V}|}\sum_{i \in \widetilde{V}}f_i(S)$$
instead of $f$ in Algorithm $4$. It can be proved that with a high probability, we can still obtain a good approximation to the optimal solution, when $f_i$'s are bounded. The accuracy of the approximation is quantified by the following theorem.

\begin{Performance}
\label{bx}
Assume that for $S\subseteq V$ and $1\le i\le n$, $|f_i(S) |\leq 1$. We uniformly choose a subset $\widetilde{V}$ from $V$, with
$$ |\widetilde{V}| \geq 2\epsilon ^{-2} b^2\left( b\log |V|+\log (2/\delta)\right),$$
and use $f_{\widetilde{V}}$ instead of $f$ in Algorithm \ref{d-3}. Then with probability of at least $1-\delta$, the output $S$ of Algorithm \ref{d-3} satisfies
$$f_{\widetilde{V}}(S) \ge \left(\frac{1}{1+2d}-\epsilon\right)(\textrm{OPT}-\epsilon).$$
\end{Performance}
Its proof follows the similar argument as the proof of Theorem 6.2 in \cite{badanidiyuru2014streaming}, where the authors deal with the submodular maximization problem under one cardinality constraint. Now we adopt a two-pass streaming algorithm for the $d$-MASK problem with ground-set dependent submodular objective functions: in the first pass, we utilize reservoir sampling \cite{vitter1985random} to sample an evaluation set $\widetilde{V}$ randomly; in the second pass, we run Algorithm \ref{d-3} with the objective function $f_{\widetilde{V}}$ instead of $f$.

\section{Applications}
In this section, we discuss two real-world applications for Algorithm \ref{d-3}: news recommendation and scientific literature recommendation.

\subsection{News Recommendation}
Nowadays, people are facing many news articles on the daily basis, which highly stresses their limited reading time. A news recommendation system helps people quickly fetch the information they need. Specifically, it provides the most relevant and diversified news to people by exploiting their behaviors, considering their reading preferences, and learning from their previous reading histories.

However, the vast amount of news articles in the dataset are hard to be processed efficiently. In \cite{raman2012online}, the authors modeled the user behavior as a submodular maximization problem. Based on the learning result, a classical greedy algorithm \cite{nemhauser1978analysis} was implemented to provide a set of relevant articles to the users. However, the large amount of data in the dataset prevents the classical greedy algorithm from producing the solution in time due to its expensive computation cost. Besides, the reading behavior of the users was oversimplified in \cite{raman2012online}, where it is assumed that each user reads a fixed number of articles per day. Since the time spent on different news articles varies, it is more reasonable to use the number of words of the articles as the measure of the reading behaviour. Hence, we can formulate this question into a $1$-MASK problem as follows:
\begin{equation*}
\begin{aligned}
& \underset{S \subseteq V}{\textrm{maximize}}
& & f(S) = \mathbf{w}^T\mathbf{F}(S) \\
& \textrm{subject to}
& & \sum_{j \in S}c_j \leq b,
\end{aligned}
\end{equation*}
where $c_j$ is the number of words in article $j$. Here $\mathbf{F}: 2^V\to [0,\infty)^m$, where $m$ is the number of features. We require the total number of words in the selected articles not to exceed a specified budget $b$, due to the limitation of the user reading time. In addition, we assume that the non-negative parameter vector $\mathbf{w}$ is learnt by a statistical learning algorithm, based on the historical user preference (three such learning algorithms can be found in \cite{raman2012online}, \cite{li2011scene}, and \cite{ijntema2010ontology}, respectively). Let $(\phi_1(d),\ldots,\phi_m(d))$ be the characteristic vector of article $d$, where for $1\le j\le m$, $\phi_j(d)=1$ if $d$ has feature $j$, $\phi_j(d)=0$, otherwise. We then define $\mathbf{F}(S)=(F_1(S),\ldots,F_m(S))$; here for $1\le j\le m$, $F_j(S)$ is the aggregation function of $S$ with respect to feature $j$ and defined by
$$F_j(S)\triangleq\log\left(1+\sum_{s\in S} \phi_j(s)\right).$$
This choice of function $F_j$ guarantees both precision and coverage of the solution set. On one hand, the monotonicity of $F_j(S)$ encourages feature $j$ to be selected if its corresponding weighting parameter $w_j$ (the $j$-th coordinate of the vector $\mathbf{w}$) is relatively large. On the other hand, the diminishing return property of $F_j$ prevents too many items with feature $j$ from being selected.

Notice that function $F_j$ is a monotone submodular function. To see this, let
$$G_j(S)\triangleq \sum_{s\in S}\phi_j(s).$$
Obviously, $G_j(S)$ is a non-decreasing modular function. With the fact that $\zeta(x)\triangleq\log(1+x)$ is an increasing concave function, we can conclude that $F_j(S)=\zeta(G_j(S))$ is a monotone submodular function. Since both monotonicity and submodularity are closed under the non-negative linear combinations \cite{fujishige2005book}, $f$ is a monotone submodular function as well. The solution based on Algorithm \ref{d-3} to this $1$-MASK problem provides the user a quick news recommendation.

As an illustration, we analyze the dataset collected in \cite{wolfe2010interaction}, which contains over $7,000$ feedback entries from $25$ people with around $8,000$ news articles. We set $m=480$ and $b=20$, with each entry of $C$ randomly chosen from a uniform distribution over $\{1,2,3,4,5\}$. The learning algorithm proposed in \cite{raman2012online} is used to calculate $\mathbf{w}$. We then compare Algorithm \ref{d-3} with the greedy algorithm in \cite{sviridenko2004note}.
\begin{figure}
   \centering
   \includegraphics[width=1\linewidth]{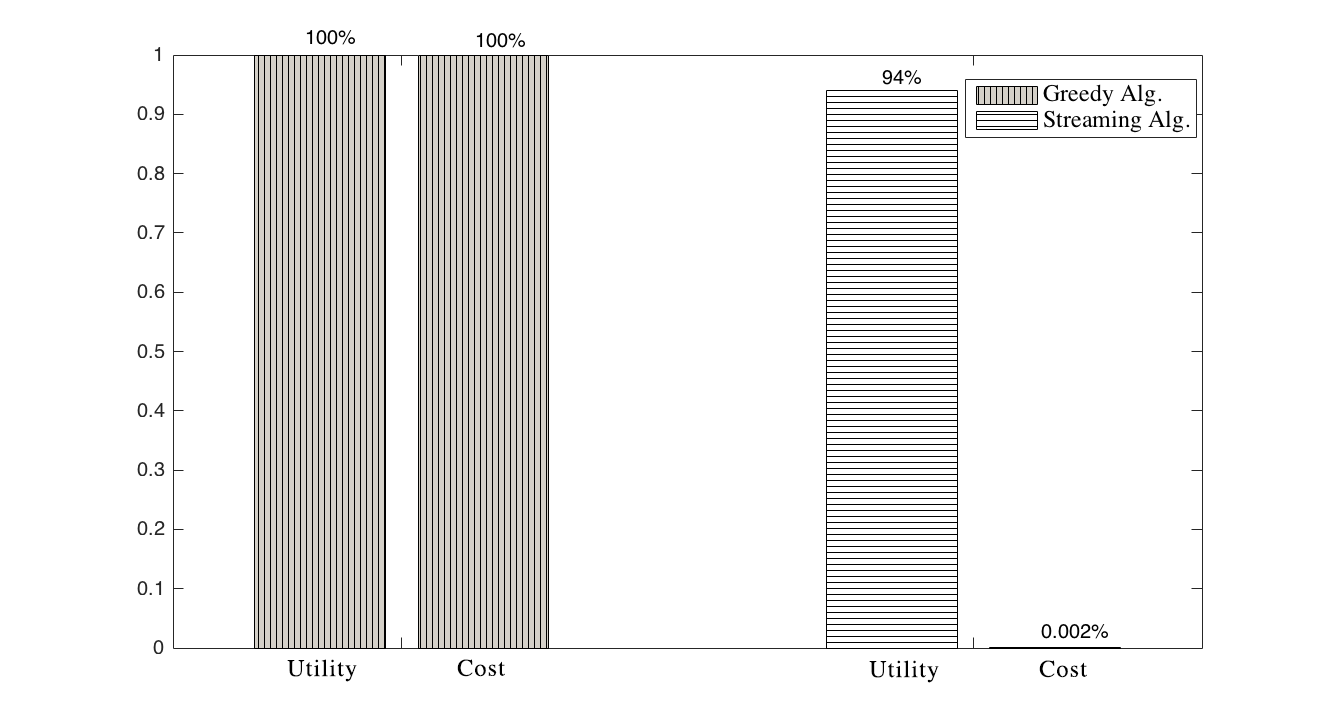}
   \caption{Comparison of Utilities and Computation Costs between the Greedy Algorithm and Streaming Algorithm}
   \label{fig:bar1}
\end{figure}

In Fig. \ref{fig:bar1}, we set the objective function value obtained by the classical greedy algorithm and its computation time both to be $1$, after using them to normalize the function value and computation time corresponding to our streaming algorithm, respectively. It has been shown that our streaming algorithm achieves $94\%$ utility of the greedy algorithm, but only requires a tiny fraction of the computation cost. Thus the proposed algorithm works well in the news recommendation system and is practically useful over large datasets.

\subsection{Scientific Literature Recommendation}
Next, we introduce an application in scientific literature recommendation. Nowadays, the researchers have to face an enormous amount of articles and collect information that they are interested in, where they have to filter the massive existing scientific literatures and pick the most useful ones. A common approach to locate the targeted literatures is based on the so-called citation networks \cite{mcnee2002recommending}. The authors in \cite{mcnee2002recommending} mapped a citation network onto a rating matrix to filter research papers.  In \cite{gori2006research}, an algorithm utilizing the random-walker properties was proposed. It transforms a citation matrix into a probability transition matrix and outputs the entries with the highest biased PageRank scores.

We here propose a new scientific literature recommendation system based on the citation networks and the newly proposed streaming algorithm (Algorithm \ref{d-3}). Consider a directed acyclic graph $G=(V,E)$ with $V=\{1,2,\ldots,n\}$, where each vertex in $V$ represents a scientific article. Let $\mathcal{R}_i$ denote the number of references contained in article $i$. The arcs between papers represent their citation relationship. For two vertices $i,j\in V$, arc $(i,j) \in E$ if and only if paper $i$ cites paper $j$. The information spreads over the reverse directions of the arcs. As an example, Fig.~\ref{fig:example} presents a citation network, which contains six vertices and seven arcs. Each of six papers cites a certain number of references.
\begin{figure}
   \centering
   \includegraphics[width=1.7in]{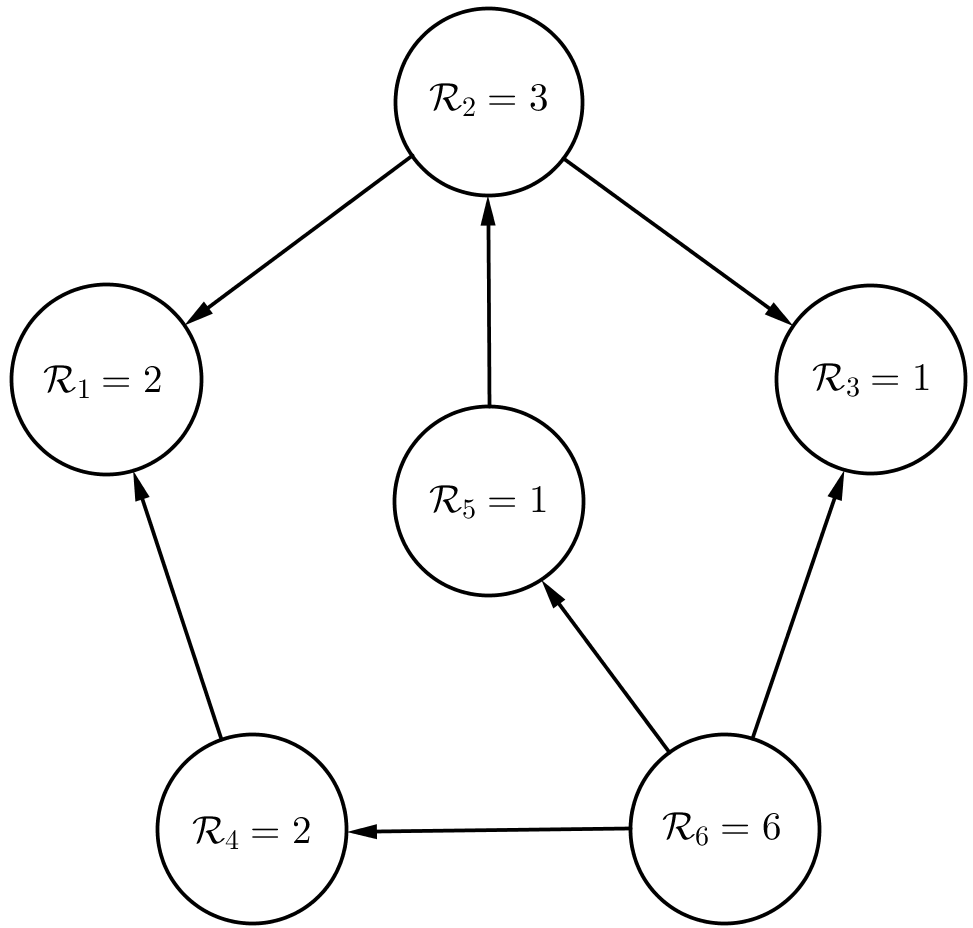}
   \caption{An Example of Citation Networks}
   \label{fig:example}
\end{figure}
The information initiates from a set of vertices (source papers), and then spreads across the network. Let $A$ be the collection of the source papers.  Our target is to select a subset $S$ out of $V$ to quickly detect the information spreading of $A$. For example, $A=\{1,3,4\}$ in Fig.~\ref{fig:example}. If we choose $S=\{6\}$, we can detect the source papers $1,3,4$ by paths $6\to 4\to 1$, $6\to 3$ and $6\to 4$, respectively. This problem can be formulated as a monotone submodular maximization under a $3$-knapsack constraint\footnote{The reason why we set $d=3$ will be explained later in this section; based on the different usages, the number of knapsack constraints and the corresponding budgets can be changed accordingly.}:
\begin{eqnarray}\label{litrec}
\begin{aligned}
   &\underset{S \subseteq V}{\textrm{maximize}} &&R(S) \\
   &\textrm{subject to} && C\bold{x}_S \leq \bold{b},
\end{aligned}
\end{eqnarray}
where $C=(c_{i,j})$ is a $3\times n$ matrix and $\bold{b}=(b_1,b_2,b_3)^T$.

Observe that the papers in $A$ transfer their influence through the citation network, but this influence becomes less as it spreads through more hops. Let $T(s,a)$ be the length of the shortest directed path from $s$ to $a$. Then the shortest path length from any vertex in $S$ to $a$ is defined as
$$T(S,a)\triangleq \min_{s\in S} T(s,a).$$
Let $W(a)$ be a pre-assigned weight to each vertex $a\in A$ such that $\sum_{a\in A}W(a)=1$. Then our goal is to minimize the expected penalty
$$\pi(S)\triangleq \sum_{a\in A} W(a) \min\{T(S,a),T_{\max}\},$$
or maximize the expected penalty reduction
$$R(S)\triangleq T_{\max}-\pi(S)=\sum_{a\in A} W(a) [T_{\max}-T(S,a)]^+,$$
where $[x]^+\triangleq \max\{x,0\}$ and $T_{\max}$ is a given maximum penalty. Note that $R$ is a monotone submodular function. To see this, for two subsets $B\subseteq C\subseteq V$, we have $T(B,a)\ge T(C,a)$ for any $a\in A$, such that $R(B)\le R(C)$; $T_{\max}-T(S,a)$ is a submodular function with respect to $S$ since
\begin{align*}
&T(B,a)-T(B\cup \{v\},a)=[T(B,a)-T(v,a)]^+ \\
&\ge [T(C,a)-T(v,a)]^+=T(C,a)-T(C\cup \{v\},a),
\end{align*}
with $v\in V\setminus C$. Then $R(S)$ is also submodular, since it is a convex combination of $T_{\max}-T(S,a)$ for $a\in A$.

We construct three constraints in (\ref{litrec}) from the aspects of recency, biased PageRank score, and reference number respectively. The first aspect is from the fact that readers prefer to read the recently published papers.  Let $c_{1,j}$ be the time difference between the publishing date of paper $j$ and the current date, and $b_1$ be the corresponding limit.

For the second aspect, the classical PageRank algorithm \cite{ilprints422} could be used to compute an important score for every vertex in the graph: a vertex will be assigned a higher score if it is connected to a more important vertex with a lower out-degree. The authors in \cite{gori2006research} introduced a so-called biased PageRank score. It is a measure of the significance of each paper, not only involving the propagation and attenuation properties of the network, but also taking the set of source vertices into account. Let $\rho(j)$ be the biased PageRank score of article $j$. We further choose a function $\xi(x)\triangleq1+\frac{1}{1+x}$ to map the PageRank score onto $(1,2]$. Then paper $j$ with the smaller value $c_{2,j}\triangleq \xi(\rho(j))$ is more valuable for the researchers. Also we set $b_2$ to be corresponding budget.

Thirdly, we assume that more references listed in the paper, more time the reader spends on picking the valuable information. Then we set $c_{3,j}$ to be the number $\mathcal{R}_j$ of references in paper $j$ and $b_3$ be the budget of the total number of references.

To evaluate the performance of Algorithm \ref{d-3},  for scientific literature recommendation, we utilize a dataset collected in \cite{joseph2007citation}. This dataset includes more than 20,000 papers in the Association of Computational Linguistics (ACL). There are two methods to evaluate the performance of an algorithm for literature recommendation: online evaluation and offline evaluation. In the online evaluation, some volunteers are invited to test the performance of the recommendation system and express their opinions. Here we use the offline evaluation to compare the function values obtained by our proposed algorithm (Algorithm \ref{d-3}) and the PageRank algorithm proposed in \cite{gori2006research}.

We perform the sensitive analysis over different knapsack constraints. With the other two constraints fixed, we change the value of the budget corresponding to the recency, biased PageRank score or reference number, respectively. Here we randomly select five nodes as the source papers. We set $T_{\max}=50$ and $W(a)=0.2$ for each source paper $a$. The results for the optimal objective values are shown in Fig. \ref{fig:b1} (with fixed $b_2=10$, $b_3=20$), Fig. \ref{fig:b2} (with fixed $b_1=20$, $b_3=20$) and Fig. \ref{fig:b3} (with fixed $b_1=20$, $b_2=10$), respectively.
It can be observed that the relative difference is around 10\% between the function values obtained by our streaming algorithm (blue lines) and the corresponding online bounds (red lines). 

\begin{figure}[H]
   \centering
   \includegraphics[width=3in]{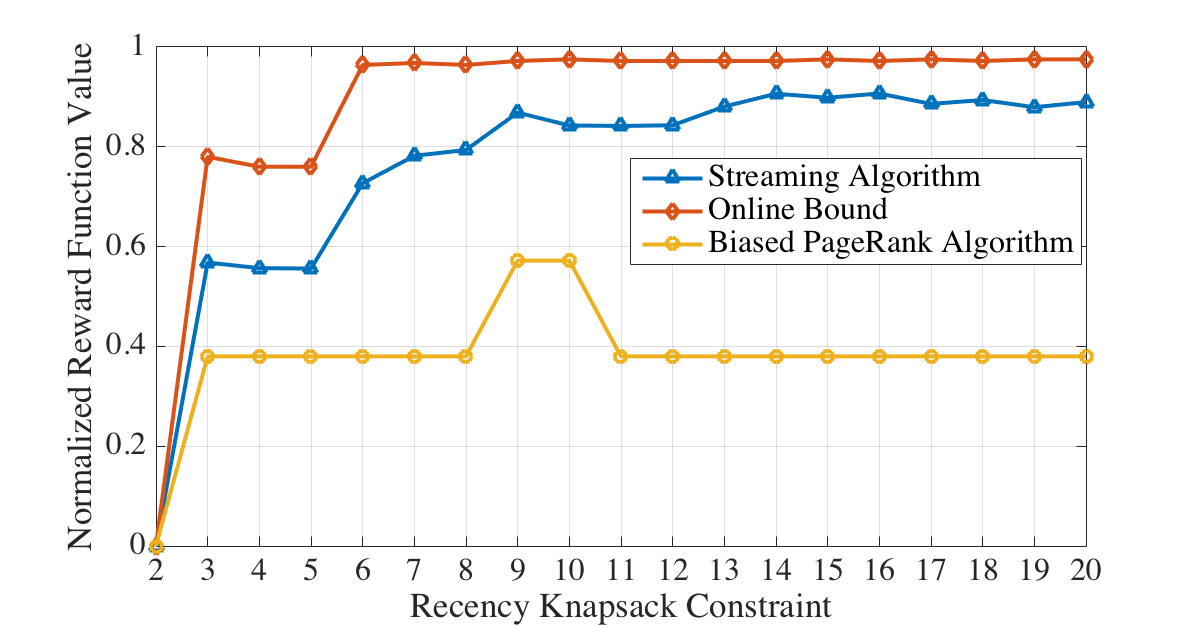}
   \caption{Optimal Function Values corresponding to Different Recency Constraints}
   \label{fig:b1}
\end{figure}
\begin{figure}[H]
  \centering
   \includegraphics[width=3in]{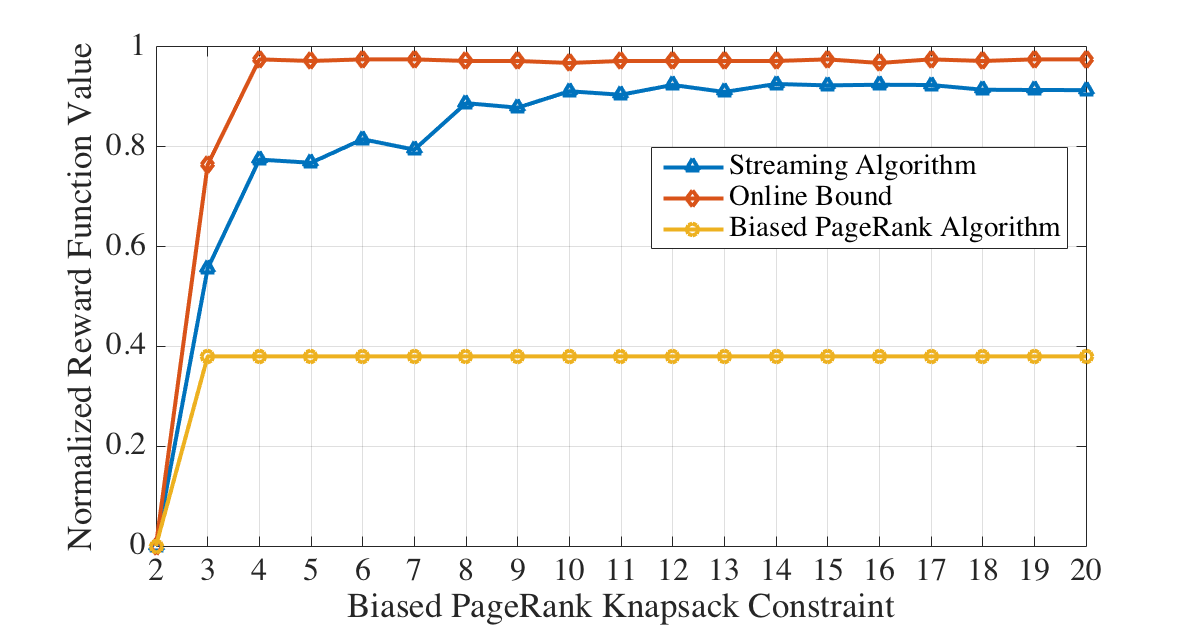}
   \caption{Optimal Function Values corresponding to Different Biased PageRank Constraints}
   \label{fig:b2}
\end{figure}
\begin{figure}[H]
   \centering
   \includegraphics[width=3in]{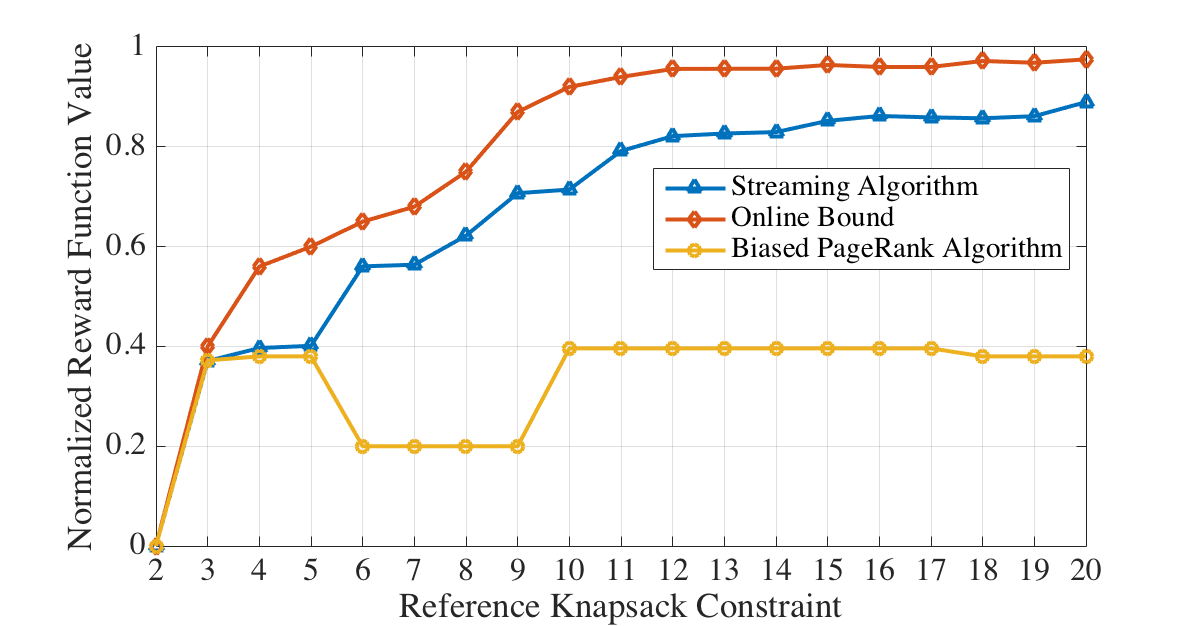}
   \caption{Optimal Function Values corresponding to Different Reference Knapsack Constraints}
   \label{fig:b3}
\end{figure}
Also, we find that our algorithm highly outperforms the biased PageRank algorithm as shown in Figs.~\ref{fig:b1}, \ref{fig:b2} and \ref{fig:b3}, respectively. Although the biased PageRank algorithm suggests the papers with high biased PageRank scores, most of the suggested papers have very long distances from the set of source articles (even disconnected from the source papers in some cases), which leads to a very low objective function value.

\section{Conclusions}
In this paper, we proposed a streaming algorithm to maximize a monotone submodular function under a $d$-knapsack constraint. It leads to a $\left(\frac{1}{1+2d} -\epsilon\right)$ approximation of the optimal value, and requires only a single pass through the dataset and a small memory size. It achieves a major fraction of the utility function value obtained by the greedy algorithm with a much lower computation cost, which makes it very practically implementable. Our algorithm provides a more efficient way to solve the related combinatorial optimization problems, which could find many good applications, such as in news and scientific literature recommendations as shown in the paper.

\section*{Acknowledgement}
We thank Dr. Nao Kakimura for comments that greatly improved the manuscript.

\end{document}